\documentclass{article}

\usepackage[preprint]{neurips_2026}

\usepackage[utf8]{inputenc} 
\usepackage[T1]{fontenc}    
\usepackage{hyperref}       
\usepackage{url}            
\usepackage{booktabs}       
\usepackage{amsfonts}       
\usepackage{nicefrac}       
\usepackage{microtype}      
\usepackage{xcolor}         
\usepackage{amsthm}

\usepackage{amssymb}
\usepackage{amsmath}
\usepackage{bm}
\usepackage[capitalise]{cleveref}
\usepackage{graphicx}

\newcommand{\Real}{\mathbb{R}}
\newcommand{\boldu}{\boldsymbol{u}}
\newcommand{\boldtheta}{\boldsymbol{\theta}}
\newcommand{\boldomega}{\boldsymbol{\omega}}
\newcommand{\mmd}{\text{MMD}^2}
\newtheorem{theorem}{Theorem}
\newtheorem*{theorem*}{Theorem}

\title{Watch your neighbors: Training statistically accurate chaotic systems with local phase space information}

\author{%
  Joon-Hyuk Ko \\
  Center for AI and Natural Sciences \\
  Korea Institute for Advanced Study \\
  Seoul, Korea 02455 \\
  \texttt{jhko725@kias.re.kr} \\
  \And
  Andrus Giraldo \\
  School of Computational Sciences \\
  Korea Institute for Advanced Study \\
  Seoul, Korea 02455 \\
  \texttt{agiraldo@kias.re.kr} \\
  \AND
  Deok-Sun Lee \\
  School of Computational Sciences \\
  Korea Institute for Advanced Study \\
  Seoul, South Korea 02455 \\
  \texttt{deoksunlee@kias.re.kr} \\
}

\begin{document}

\maketitle

\begin{abstract}
  Chaotic systems pose fundamental challenges for data-driven dynamics discovery, as small modeling errors lead to exponentially growing trajectory discrepancies. Since exact long-term prediction is unattainable, it is natural to ask what a good surrogate model for chaotic dynamics is. Prior work has largely focused either on reproducing the Jacobian of the underlying dynamics, which governs local expansion and contraction rates, or on training surrogate models that reproduce the ground-truth dynamics' long-term statistical behavior.  In this work, we propose a new framework that aims to bridge these two paradigms by training surrogate dynamics models with accurate Jacobians and long-term statistical properties. Our method constructs a local covering of a chaotic attractor in phase space and analyzes the expansion and contraction of these coverings under the dynamics. The surrogate model is trained by minimizing the maximum mean discrepancy between the pushforward distributions of the coverings under the surrogate and ground-truth dynamics. Experiments show that our method significantly improves Jacobian accuracy while remaining competitive with state-of-the-art statistically accurate dynamics learning methods. Our code is fully available at \url{https://anonymous.4open.science/r/neighborwatch}.
\end{abstract}

\section{Introduction}

Modeling dynamical systems and generating accurate forecasts are central problems across the sciences. When governing equations are unavailable or intractable, data-driven discovery offers an alternative by training surrogate models from trajectory data. Typically, model parameters are learned by minimizing discrepancies between observed and predicted trajectories, yielding “black-box” models that can match or even outperform approximate theoretical descriptions \citep{rackauckas2021,brunton2022,north2023,yu2024}.

Chaotic systems pose unique challenges for data-driven dynamics as the trajectories of these systems are extremely sensitive to variations in the initial condition and model parameters. As such, unavoidable errors or noise in trajectory measurements or model parameters result in trajectory predictions that err exponentially in time, rendering long-term exact trajectory predictions fundamentally impossible. While one can alternatively optimize models based on their short-term predictions alone, models trained this way 
are myopic, quickly outputting unrealistic or diverging solutions when rolled out into the future \citep{li2022,hess2023,schiff2024,li2025}. This incompatibility between trajectory-based optimization and chaotic dynamics necessitates a reevaluation of what constitutes a good surrogate model for chaos.

Recent works in the machine learning community have focused on training models that are "statistically accurate", meaning that even though the long-time model trajectory predictions cannot match that of the ground truth, expectation values computed along these trajectories must agree \citep{jiang2023,platt2023,schiff2024}. The validity of such methods stems from the fact that many chaotic systems possess physical measures that are invariant with respect to their dynamics and are ergodic inside their chaotic attractor. Therefore, long-time trajectory statistics are invariant, global properties of these chaotic dynamics, serving as well-defined targets for the dynamics discovery task.

An alternate line of thought can be found in the dynamical systems literature, where the goal was to train surrogate models that can accurately reproduce the Jacobian of the ground-truth dynamical equations \citep{abarbanel1992,nychka1992,mccaffrey1992,escot2023}. This was largely motivated by the fact that Jacobians provide information about contracting and repelling dynamics of the phase space, which in turn give rise to the key characteristics of chaos, including the sensitivity to initial trajectory perturbations \citep{gilmore2011}. From an alternative standpoint, the recent work of Eisen et al. \citep{eisen2025} has also highlighted the importance of learning accurate Jacobians for chaotic systems, for its pivotal role in identifying interacting subdomains in complex systems such as the brain.

In this work, we aim to bridge the two schools of thought by developing a training method that yields surrogate models for chaotic systems with accurate Jacobians and long-term statistical properties. 
Our work succeeds that of Park et al. \citep{park2025}, where the connection between conditioning models with Jacobian information and their capability to reproduce accurate long-term statistical behaviors was also explored. However, their method relied on an oracle supplying the value of the ground-truth Jacobians, making it unsuitable for real-world applications. 

\paragraph{Contributions} 
We propose a novel framework to train Neural ODEs on chaotic data that are statistically accurate and properly capture the local expansion and contraction of the chaotic attractor. We achieve this by focusing on how local sets evolve in time and developing a specialized loss function that regularizes training with this information. By explicitly accounting for noise during algorithm development, our experiments show that the method yields greatly improved Jacobians, even for noisy data, while retaining statistical accuracy. This setting also better reflects realistic applications, where measurements are corrupted by observational noise and the ground-truth dynamics are unavailable, thereby demonstrating the robustness and practical versatility of our approach.
\section{Background}
\label{gen_inst}

We consider dynamical systems of the form
\begin{equation}
    \displaystyle \frac{d \boldu}{d t} = f(\boldu(t)),
    \label{eqn: autonomous dynamics}
\end{equation}
where \(\boldu (t) \in \Real^d\) is the time-dependent state vector, and $f: \Real^d  \rightarrow \Real^d$ is a sufficiently smooth function. \cref{eqn: autonomous dynamics} encompasses a wide variety of systems, including autonomous ordinary differential equations (ODEs), non-autonomous ODEs,\footnote{By introducing an augmented state vector \(\boldsymbol{v}=[t,\boldu]^\intercal\) \citep{strogatz2024}.} and spatial discretization of partial differential equations (PDEs).

Associated with \cref{eqn: autonomous dynamics} is the time-$T$ map $\phi_T: \mathbb{R}^d \rightarrow \mathbb{R}^d$, defined by $\phi_T(\boldu_0)$ as the state of the system obtained by evolving the initial condition $\boldu_0$ for time $T$ \citep{guckenheimer1983}. That is, for any solution $\boldu(t)$ with 
$\boldu(0)=\boldu_0 \in \mathbb{R}^d$ of \cref{eqn: autonomous dynamics}, $\phi_T(\boldu_0)$ is effectively
\begin{equation*}
    \phi_T(\boldu_0):=\boldu(T)=\boldu_0 + \int_{0}^{T} f\left(\boldu(t)\right) dt.
\end{equation*}
Since $f$ is smooth, then $\phi_T$ is a diffeomorphism. Notice also that under this notation, the time-$nT$ map, for $n\in\mathbb{N}$ satisfies \(\phi_{nT}=\phi_T \circ ...\circ \phi_T =\phi^n_T;\)
that is, it is $n$-fold composition of the time-$T$ map.

We are primarily interested in dissipative systems with a compact attracting invariant set $\mathcal{A} \subset \mathbb{R}^d$. Within this class, we focus on chaotic attractors, for which the dynamics on $\mathcal{A}$ exhibit sensitivity to initial conditions and aperiodic behavior; see \cite{alligood1996} for a more formal introduction. 

In particular, we consider chaotic attractors that support a physical measure $\mu$ \citep{Young2002}.  Any time $T$-map is measurable under $\mu$, and $\mu$ is invariant under any time-$T$ map, i.e., 
\begin{equation*}
\mu(\phi^{-1}_T(B))=\mu(B), \quad \text{for all } T>0 \text{ and all Borel sets } B\subset \mathbb{R}^d.
\end{equation*}
Furthermore, there exists a positive Lebesgue measure set $V \subset \mathbb{R}^d$ such that for any continuous function $g:\mathbb{R}^d\rightarrow \mathbb{R}$,
\begin{equation*}
\lim_{n\rightarrow \infty} \frac{1}{n}\sum^{n-1}_{i=0} g(\phi^{i}_T(x)) = \int_{\mathcal{A}} g d\mu
\end{equation*}
for every $x\in V.$ Since the support of $\mu$ is $\mathcal{A}$, it follows that $\mu(\mathcal{A}) = 1$ and $\mu(\mathbb{R}^d \setminus \mathcal{A}) = 0$. Associated with $\mu$ and the map $\phi_T$ is the pushforward measure $(\phi_T)_\# \mu$, defined on the Borel $\sigma$-algebra of $\mathbb{R}^d$ by
\begin{equation*}
(\phi_T)_\# \mu (B) = \mu\big(\phi_T^{-1}(B)\big),
\end{equation*}
for all Borel sets $B \subset \mathbb{R}^d$. The pushforward $(\phi_T)_\# \mu$ is again a measure. Moreover, since $\mu$ is invariant under $\phi_T$, it follows that $(\phi_T)_\# \mu = \mu$.
Consequently, for every $n \in \mathbb{N}$, we have $(\phi_{nT})_\# \mu = (\phi_T^n)_\# \mu = \mu$; that is, the $\mu$ is invariant under all iterates of the time-$T$ map.

\section{Learning accurate surrogate dynamics models for chaos}

In this work, we focus on the problem of training surrogate dynamics models for chaotic systems using trajectory measurements on the chaotic attractor. One powerful framework for learning continuous-time dynamical systems is the neural ordinary differential equation (Neural ODE) \citep{chen2018,kidger2022}, where a neural network \(f_{\boldtheta}:\Real^d\rightarrow\Real^d\) with learnable parameters \(\boldtheta\in\Real^p\) is used to model the dynamics as,
\begin{equation}
    \displaystyle \frac{d \tilde{\boldu}}{d t} = f_{\boldtheta}(\tilde{\boldu}).
    \label{eqn: neural ode}
\end{equation}
Neural ODEs are particularly suited for learning continuous time dynamics from data because not only do they satisfy the differential equation structure of the ground truth dynamics a priori, they can also be augmented with additional prior knowledge, such as conserved quantities \citep{greydanus2019,cranmer2020} or partially known ground truth equations \citep{rackauckas2021}.

Given a training dataset consisting of \(n\) trajectories sampled at discrete measurement times \(\{t_j\}_{j=0}^{m}\), such that  \( \mathcal{D}_{train}=\{\{\boldu_j^{(i)}\}_{j=0}^m\}_{i=1}^n\), Neural ODE training is typically formulated as solving the following unconstrained optimization problem:
\begin{equation}
    \min_{\boldtheta} \mathcal{L}_{traj}(\boldtheta);\quad \mathcal{L}_{traj}(\boldtheta) = \frac{1}{S\cdot|\mathcal{D}_{train}|}\sum_{(i,j)\in \mathcal{D}_{train}}\sum_{s=1}^S||\phi_{s\Delta t;\boldtheta}(\boldu_j^i)-\boldu_{j+s}^i
   ||^2_2,
    \label{eqn: trajectory loss}
\end{equation}
where \(\phi_{T;\boldtheta}\) is the time-\(T\) map of the Neural ODE, calculated by integrating \cref{eqn: neural ode} with a differentiable numerical ODE solver. Notice that $\phi_{\Delta t}(\boldu_j^i)=\boldu^i_{j+1}$.

In practice, \(\mathcal{L}_{traj}(\boldtheta)\) is optimized using mini-batch gradient descent \citep{goodfellow2016} --in each training step, length \(S\) trajectory segments are randomly sampled from \(\mathcal{D}_{train}\) to produce a batch \(\mathcal{B}\), with which the loss is calculated and optimized for.

Statistically accurate learning methods for chaotic dynamics attempt to regularize the unstable learning problem of \cref{eqn: trajectory loss} by solving the constrained problem 
\begin{equation}
    \min_{\boldtheta} \mathcal{L}_{traj}(\boldtheta) \quad s.t.\quad \mu_{\boldtheta}=\mu,
    \label{eqn: statisticaly accurate objective}
\end{equation}
where \(\mu_{\boldtheta}\) is the physical measure of the Neural ODE dynamics. As it is difficult to strictly enforce the constraint during optimization, \cref{eqn: statisticaly accurate objective} is often relaxed into the following Lagrange multipliers form:
\begin{equation}
    \min_{\boldtheta} \mathcal{L}(\boldtheta); \quad \mathcal{L}(\boldtheta)=\mathcal{L}_{traj}(\boldtheta) + \lambda \cdot \ell(\mu_{\boldtheta},\mu),
    \label{eqn: dyslim-like loss}
\end{equation}
where $\ell$ is a suitable pairwise loss function that compares both measures. In practice, directly computing such losses is often intractable, and one instead compares observables or statistics induced by the measures. Examples include Lyapunov exponents \citep{platt2023}, domain-knowledge-based collections of summary statistics \citep{jiang2023}, or distances between empirical trajectory distributions in phase space, such as the maximum mean discrepancy (MMD) \citep{schiff2024}.

\section{Neighborhood-based training of neural ODEs}

Our goal is to modify and extend the statistically accurate learning framework, so that the resulting Neural ODEs also accurately capture the local expansion and contraction of the chaotic attractor. Indeed, while the objective function in \cref{eqn: dyslim-like loss} enforces agreement between invariant measures, local features of how the flow transports and deforms sets might be lost when implementing \cref{eqn: dyslim-like loss}.

To achieve our goal, we shift from a global to a local perspective by studying how the pushforward measure distributes over the images of sets forming a cover of the attractor under the dynamics. That is, we aim to solve the following constrained optimization problem instead:
\begin{equation}
    \min_{\boldtheta} \mathcal{L}_{traj}(\boldtheta) \quad s.t.\quad (\phi_{\Delta t;\boldtheta})_\#\left.\mu\right|_{\phi_{\Delta t}(U_i)}=(\phi_{\Delta t})_\#\left.\mu\right|_{\phi_{\Delta t}(U_i)},\ \forall i,
    \label{eqn: local objective}
\end{equation}
where $\{U_i\}_{i=1}^{P}$ is a finite cover of the chaotic attractor $\mathcal{A}$ by Borel measurable sets, and
\[
\left. (\phi_{\Delta t})_\#\mu \right|_{\phi_{\Delta t}(U_i)}
\quad \text{and} \quad
\left. (\phi_{\Delta t;\boldtheta})_\#\mu \right|_{\phi_{\Delta t}(U_i)}
\]
denote the restrictions of the corresponding pushforward measures to the set $\phi_{\Delta t}(U_i)$. Note that $\{\phi_{\Delta t}(U_i)\}_{i=1}^{P}$ is also a finite cover of $\mathcal{A}$, since $\phi_{\Delta t}$ is a diffeomorphism and $\mathcal{A}$ is \(\phi_{\Delta t}\)-invariant.

The constraint proposed in \cref{eqn: local objective} requires the Neural ODE to reproduce how the ground truth dynamics transports the physical measure $\mu$, when restricted to each set (i.e. neighborhood) of the cover, over time intervals of length $\Delta t$.

At the same time, our constraint also guides Neural ODEs to learn the global physical measure of the data, which follows from the theorem below (see proof in \cref{appendix: theorem proof}):
\begin{theorem}\label{theo:1}
Let $(X,\mathcal{C})$ be a measurable space, and let $\nu_1,\nu_2$ be finite measures on $(X,\mathcal{C})$ such that $ \nu_1(X\setminus \mathcal{A})=\nu_2(X\setminus \mathcal{A})=0 $
for a set $\mathcal{A}\in\mathcal{C}$ (both measures have support on $\mathcal{A}$). Furthermore, let $\{B_j\}_{j\in\mathbb{N}}\subset\mathcal{C}$ be a countable cover of $\mathcal{A}$, and assume that $\nu_1|_{B_j}=\nu_2|_{B_j}$, $\forall j$; that is
\[
\nu_1(C)=\nu_2(C)\quad\text{for all measurable }C\subset B_j,\;\forall j.
\]
Then $\nu_1=\nu_2$ on $\mathcal{C}$.
\end{theorem}
For our purposes, $\nu_1=(\phi_{\Delta t})_\#\mu, \, \nu_2=(\phi_{\Delta t;\boldtheta})_\#\mu$ and $\{B_j\}_{j\in\mathbb{N}}=\{\phi_{\Delta t}(U_j)\}_{j\in\mathbb{N}}$.
Since $\mu$ is invariant under the ground-truth dynamics, the above theorem implies that \((\phi_{\Delta t;\boldtheta})_\# \mu =(\phi_{\Delta t})_\#\ \mu  = \mu\). Consequently, the Neural ODE also preserves the physical measure under iterates of the time-$\Delta t$ map. For sufficiently small $\Delta t$, this suggests that the Neural ODE dynamics approximately preserve $\mu$ over continuous time intervals as well.

To transform \cref{eqn: local objective} into a concrete algorithm for Neural ODE training, we need three ingredients: 
\begin{enumerate}
    \item A strategy for constructing a finite cover $\{U_i\}_{i=1}^P$ (and their image  $\{\phi_{\Delta t}(U_i)\}_{i=1}^P$) of the attractor and estimating the pushforward measures generated by the ground-truth dynamics on the images of the cover sets: $\left. (\phi_{\Delta t})_\# \mu \right|_{\phi_{\Delta t}(U_i)}.$ 
    \item  A procedure for evolving the cover sets under the Neural ODE dynamics and estimating the corresponding pushforward measures, when restricted to the cover $\{\phi_{\Delta t;\boldtheta}(U_i)\}_{i=1}^P\}$:$\left. (\phi_{\Delta t;\boldtheta})_\# \mu \right|_{\phi_{\Delta t}(U_i)}.$
    \item A pairwise loss function quantifying the discrepancy between the Neural ODE and ground-truth pushforward measures on each evolved neighborhood: $\ell\!\left(
    \left. (\phi_{\Delta t;\boldtheta})_\# \mu \right|_{\phi_{\Delta t;\boldtheta}(U_i)},
    \left. (\phi_{\Delta t})_\# \mu \right|_{\phi_{\Delta t}(U_i)}
    \right)$.
\end{enumerate}
Notice that our construction compares two different evolved covers, $
\{\phi_{\Delta t}(U_i)\}_{i=1}^P$ and $\{\phi_{\Delta t;\boldtheta}(U_i)\}_{i=1}^P,$ together with their corresponding pushforward measures and restrictions. At first glance, this appears inconsistent with Theorem~\ref{theo:1}, since it assumes a common cover. However, the loss function $\ell$ is chosen so that
\[
\ell = 0
\quad \Longrightarrow \quad
\phi_{\Delta t}(U_i)
=
\phi_{\Delta t;\boldtheta}(U_i),
\qquad
i=1,\dots,P.
\]
Consequently, whenever the loss vanishes (or is sufficiently small), the two evolved covers coincide (or are close), allowing the corresponding restricted pushforward measures to be compared on the same sets. We detail each ingredient below.

\begin{figure}[tp]
    \centering
    \includegraphics[width=0.9\linewidth]{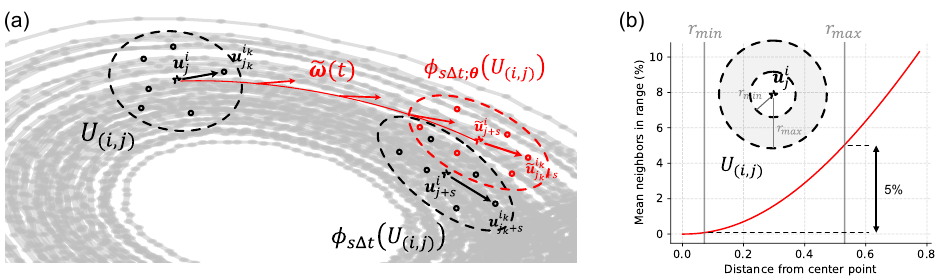}
    \caption{Schematic of our Neighborhood method. \textbf{(a)} We focus on finite cover of the attractor \(U_{(i,j)}\) and regularize training so that the model deforms the cover in the same manner as the ground truth dynamics. \textbf{(b)} We design our cover sets to be annuli to mitigate the effects of noise.}
    \label{fig: neighborhood schematic}
\end{figure}

\paragraph{Estimating the time-evolution of attractor covers from data} 
Since our dataset consists of discrete trajectory measurements on the chaotic attractor, we construct a cover $\{U_{(i,j)}\}_{i=1,j=0}^{n,m}$ by considering neighborhoods centered at the data points $\boldu_j^i$. Assuming the trajectories are sampled from the asymptotic dynamics on the attractor, the neighboring points contained in a set $U_{(i,j)}$ can be regarded as samples from the physical measure restricted to that neighborhood:$\{\boldu_{j_k}^{i_k}\} \sim \left. \mu \right|_{U_{(i,j)}},$
where $(i_k,j_k)$ denote the trajectory and time indices of the $k$-th neighboring point in $U_{(i,j)}$.

The evolution of this local neighborhood can then be approximated by evolving each neighboring point forward in time. Consequently, $
\{\boldu_{j_k+s}^{i_k}\}
\sim
(\phi_{s\Delta t})_\#
\left. \mu \right|_{\phi_{s\Delta t}(U_{(i,j)})}$ is a sample of the pushforward measure when restricted to the evolved neighborhood.

While a simple choice for the neighborhood around \(\boldu_j^i\) would be to consider an open ball of radius \(r_{max}\), for real-world applications, noise imposes an additional constraint by obscuring the time evolutions of neighbors that are too close to the center point. To counter this, we modify the neighborhood to be an annulus with inner and outer radii of \(r_{min}\) and \(r_{max}\), where \(r_{min}\) is chosen as a constant multiple\footnote{We used 8 in our experiments.} of the estimated standard deviation of the data noise. For \(r_{max}\), we find that each neighborhood needs to have some fraction of the total number of points in it; otherwise, the learned local information tends to overfit on the few neighbors watched during training. Therefore, we select \(r_{max}\) so that on average, each neighborhood contains 5\% of the total points in the dataset; see \cref{fig: neighborhood schematic}b. 

We comment that the phase space neighbors can be identified efficiently using KDTrees \citep{maneewongvatana1999}. Furthermore, as neighbor identification is performed once per dataset, prior to training, this procedure has negligible impact on the actual model training and evaluation.

\paragraph{Evolving phase space neighborhoods with Taylor expansion of the Neural ODE}
Instead of directly evolving the neighboring points $\phi_{\Delta t;\boldtheta}(\boldu_{j_k}^{i_k})$, we evolve perturbations relative to the center point $\boldu_j^i$ of the neighborhood and reconstruct the evolution of the neighborhood from these perturbations; see \cref{fig: neighborhood schematic}a. In this way, we aim to capture how small perturbations are propagated by the Neural ODE dynamics. Moreover, by constraining the evolution of perturbations relative to a common center trajectory, the method indirectly regularizes the local geometry of the learned flow, enforcing coherent local deformation of nearby points and preventing spurious local distortions in the vector field.

More precisely, the evolution of phase-space neighborhoods under the Neural ODE dynamics can be approximated through the Taylor expansion
\begin{equation}
    \frac{d}{d t}\begin{bmatrix}\tilde{\boldu}\\\tilde{\boldomega}^k\end{bmatrix}=\begin{bmatrix}
        f_{\boldtheta}(\tilde{\boldu})\\ \left. \partial f_{\boldtheta}\right|_{\tilde{\boldu}}[\tilde{\boldomega}^k]+
        \left. \frac{1}{2}\partial^2 f_{\boldtheta}\right|_{\tilde{\boldu}}[\tilde{\boldomega}^k, \tilde{\boldomega}^k]
    \end{bmatrix}, \quad \begin{bmatrix}
        \tilde{\boldu}(0)=\boldu_j^i\\
         \tilde{\boldomega}^k(0)=\boldu_{j_k}^{i_k}-\boldu_j^i,
    \end{bmatrix}
    \label{eqn: neighborhood neural ode}
\end{equation}
where \(\tilde{\boldomega}^k\) is the time-evolving perturbation vector between the center \(\boldu_j^i\) and its \(k\)-th neighbor \(\boldu_{j_k}^{i_k}\) and the second order term accounts for the fact that these perturbations are not infinitesimal. 

From this, samples from the Neural ODE-evolved cover \(U_{(i,j)}\) can be obtained as: \(\{\tilde{\boldu}_{j_k +s}^{i_k}\} \sim (\phi_{s\Delta t;\boldtheta})_{\#}\left.\mu\right|_{\phi_{s\Delta t;\boldtheta}(U_{(i,j)})}\), where \(\tilde{\boldu}_{j_k +s}^{i_k}=\tilde{\boldu}_j^i+\tilde{\boldomega}^k\).

\paragraph{Matching neighborhood predictions with the maximum mean discrepancy}

To measure the discrepancy between the Neural ODE and ground truth versions of the time-evolved local neighborhoods using point samples, we use the squared maximum mean discrepancy (MMD\textsuperscript{2}) \citep{gretton2012}:
\begin{equation}
\begin{aligned}
&\ell\left((\phi_{s\Delta t;\boldtheta})_\#\left.\mu\right|_{\phi_{\Delta t;\boldtheta}(U_{(i,j)})},(\phi_{s\Delta t})_\#\left.\mu\right|_{\phi_{\Delta t}(U_{(i,j)})}\right)=\widehat{\text{MMD\textsuperscript{2}}}(\{\boldu_{j_k +s}^{i_k}\}, \{\tilde{\boldu}_{j_k +s}^{i_k}\})\\&=\frac{1}{K(K-1)}\sum_{k=1}^K\sum_{k'\neq k}^K\left[ \kappa\left(\tilde{\boldu}_{j_k}^{i_k},\tilde{\boldu}_{j_{k'}}^{i_{k'}}\right)+\kappa\left(\boldu_{j_k}^{i_k},\boldu_{j_{k'}}^{i_{k'}}\right)\right]-\frac{2}{K^2}\sum_{k=1}^K\sum_{k'=1}^K \kappa\left(\tilde{\boldu}_{j_k}^{i_k},\boldu_{j_{k'}}^{i_{k'}}\right)
\end{aligned}
\label{eqn: mmd}
\end{equation}
where $\kappa$ is a suitable \textbf{characteristic} kernel function (see \cref{appendix: evaluation metrics} for further details). In particular, because $\kappa$ is characteristic, the condition $\widehat{\mathrm{MMD}^2}\approx 0$ implies that the corresponding pushforward measures are approximately equal. Moreover, since the loss is computed from samples supported on the evolved neighborhoods, this also enforces approximate agreement between the corresponding neighborhood images. Consequently, the setting of Theorem~\ref{theo:1} is approximately recovered.

With \cref{eqn: mmd}, we can formulate our loss function, corresponding to a relaxed version of \cref{eqn: local objective}:
\begin{equation}
    \mathcal{L}(\boldtheta)=\mathcal{L}_{traj}(\boldtheta)+\lambda \cdot\frac{1}{S\cdot|\mathcal{D}_{train}|}\sum_{(i,j)\in \mathcal{D}_{train}}\sum_{s=1}^S\ell\left((\phi_{s\Delta t;\boldtheta})_\#\left.\mu\right|_{U_{(i,j)}},(\phi_{s\Delta t})_\#\left.\mu\right|_{U_{(i,j)}}\right),
    \label{eqn: neighborhood loss function}
\end{equation}
where \(\lambda\) is a hyperparameter that balances the relative influence of the two losses. 

Analogously to \cref{eqn: trajectory loss}, our loss function is also minimized using mini-batch gradient descent. For this, we set a fixed batch size \(|\mathcal{B}|\), and sample \(B'\) center points and \(K\) neighbor points per cover set such that \(|\mathcal{B}|=B'(K+1)\). We set the number of neighbors \(K\) as a hyperparameter of our method, along with the neighborhood loss weight \(\lambda\).

\section{Experiments}
\label{others}

\subsection{Datasets}
We performed experiments on three different chaotic systems with varying dimensionalities: the 3D Lorenz 63 system \citep{lorenz1963}, the 4D Hyperchaotic Chen system \citep{li2005}, and the 6D Lorenz 96 system \citep{lorenz1996}. To account for the different time scales across the systems considered, the sampling time interval was set to be \(0.01\tau\), where the characteristic timescale \(\tau\) of each system was estimated from the averaged Fourier spectrum of long-term trajectories. For each system, the training, validation, and test data all consisted of 50 trajectories with 1000 time points each, which were scaled and shifted to have zero mean and unit standard deviation along each dimension. Subsequently, the training and validation data were corrupted using Gaussian noise with standard deviations of 0.01, 0.05 or 0.1, corresponding to 1\%, 5\%, or 10\% noise strength as computed by the ratio between noise and data standard deviations. For additional details on data generation, we refer readers to \cref{appendix: data generation}.

\subsection{Baselines}
To test the efficacy of our proposed \textbf{Neighborhood} method, we perform experiments against the following two baselines: (i) \textbf{Vanilla}: which corresponds to the standard trajectory-based Neural ODE training (\cref{eqn: trajectory loss}), and (ii) \textbf{DySLIM} \citep{schiff2024}: a method designed to improve the long-term statistical accuracy of Neural ODEs by adding two loss terms that measure the discrepancy between the invariant measures of the data and the model in a global manner. Hyperparameters of DySLIM and our Neighborhood method\footnote{Both methods have two hyperparameters each.} were determined via a small grid search prior to the main experiments, which we detail in \cref{appendix: hyperparameter selection}.

\subsection{Model training and evaluation}
For all experiments, we use multi-layer perceptrons for the Neural ODE vector field \(f_{\boldtheta}\), with nodes per layer being [\(d\), 64, 64 \(d\)]. The time-\(\Delta t\) maps of the Neural ODEs were approximated using the 5-th order adaptive Runge-Kutta solver \texttt{Tsit5} \citep{tsitouras2011} from the \texttt{diffrax} library
\citep{kidger2022}, with relative and absolute tolerance values of 1e-4 and 1e-6. We use ten-step model rollouts and train using the adabelief optimizer \citep{zhuang2020}, with a batch size of 2048 and a learning rate value of 2e-3. The number of training steps was set to 10000 for the Lorenz96 system, and 5000 for the rest; moreover, the model with the lowest validation loss at each training step was kept for subsequent evaluations. All computations --including dataset generation, model training, and evaluation-- were performed on an internal cluster equipped with NVIDIA V100 GPUs with 32GB VRAM. Each experiment was repeated three times, with different random seeds for model weight initialization. We report the associated means and standard deviations in our results whenever applicable.

\section{Results}
\subsection{Short-term trajectory predictions}
As only long-term trajectories are impossible to predict exactly, properly trained chaotic models should still be able to generate accurate short-term forecasts. Therefore, we first check the short-term prediction accuracy of the trained models using the valid prediction time (VPT), which is the time (scaled by the maximum Lyapunov exponent \(\lambda_1\)) elapsed until the normalized root mean squared error (NRMSE) between the data and the model trajectory exceeds a prespecified threshold value (\cref{fig: short-term prediction}a, \cref{appendix: evaluation metrics}).

\begin{figure}[ht]
    \centering
    \includegraphics[width=1.0\linewidth]{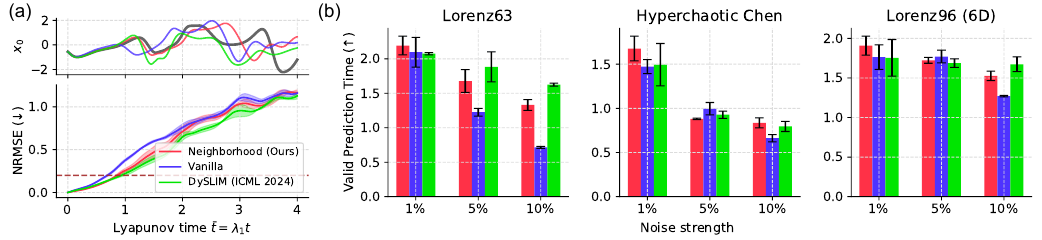}
    \caption{Short-term accuracy of the trained models. \textbf{(a)} Data  (gray line) and model predictions (trained on 10\% noise data) for the first state variable of the 6D Lorenz 96 system (upper panel) 
    and the corresponding NRMSE as a function of Lyapunov time (lower panel). 
    Dashed brown line marks the VPT threshold value (0.3) used in (b). \textbf{(b)} Valid prediction time (VPT) across different systems and training data noise strengths.}
    \label{fig: short-term prediction}
\end{figure}

From the results in \cref{fig: short-term prediction}b, we see that for lower noise strengths, conventional trajectory-based training (Vanilla) performs competitively, indicating that for short-term predictions, regularizing the MSE objective is not necessary if the data is not too degraded. 
As noise increases, all three training methods result in lower short-term accuracies, as the dynamics information in the training data becomes increasingly blurred. However, the relative performance drop between the methods differs, with the VPTs of models trained with the Vanilla method quickly dropping behind models trained with DySLIM and our Neighborhood method. This indicates that our neighborhood loss term (as well as the two additional loss terms of DySLIM) properly regularize the learning problem, preventing the model from overfitting on the noisy short-term trajectory data.

\subsection{Long-term statistical accuracy}
Next, we evaluate the long-term statistical accuracy of the differently trained models by quantifying how well the learned dynamics preserve the chaotic attractor over long periods of time. To do so, we take a sample of 5000 points on the data attractor and evolve them in time with both the ground truth equations and the learned Neural ODEs. The resulting long-time evolution of the sampled attractors is then compared visually (\cref{fig: statistical accuracy}a) and quantitatively, using the squared maximum mean discrepancy (\(\mmd\); \cref{eqn: mmd}) (\cref{fig: statistical accuracy}b, c).

\begin{figure}[ht]
    \centering
    \includegraphics[width=1.0\linewidth]{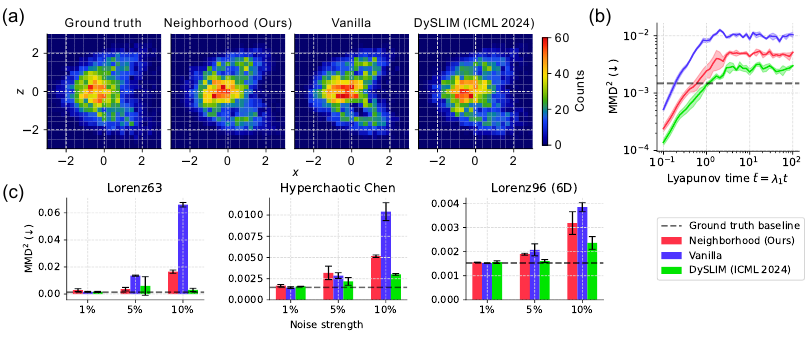}
    \caption{Long-term statistical accuracy of the trained models. \textbf{(a)} 2D histograms of the time-evolved samples of the hyperchaotic Chen system, at Lyapunov time \(\bar{t}=100\). The models depicted were trained on the 10\% noise data. \textbf{(b)} The corresponding \(\mmd\) as a function of rollout time. The gray dashed line indicates the \(\mmd\) between two independent samples of the same data attractor (see \cref{appendix: evaluation metrics} for details). \textbf{(c)} \(\mmd\) between the ground truth and model-evolved samples at \(\bar{t}=100\) for different systems and noise strengths.} 
    \label{fig: statistical accuracy}
\end{figure}

As shown in \cref{fig: statistical accuracy}b, the \(\mmd\) starts from zero and rises with time, as the ground
truth and model dynamics evolved points start from the same initial conditions and become increasingly decorrelated. After sufficient time, the \(\mmd\) plateaus to some fixed value; as the model evolves, points settle onto the model invariant measure, which may differ from that of the data (\cref{fig: statistical accuracy}a).

The results in \cref{fig: statistical accuracy}c paint a similar picture as in the short-term accuracy case (\cref{fig: short-term prediction}b), where all methods, including Vanilla, are able to learn the ground truth attractor when only a small amount of noise is present in the data. For larger noise scenarios, however, the invariant measure discovered by Vanilla becomes increasingly inaccurate, whereas both our Neighborhood and DySLIM methods successfully mitigate the effect of noise and produce more accurate invariant measures. 


While our method performs competitively with DySLIM in short-term prediction, DySLIM achieves slightly higher long-term statistical accuracy, likely due to its direct optimization of physical measure agreement in the long term. In contrast, our method matches local covers and their pushforward measures, thereby indirectly capturing both statistical and local dynamical properties simultaneously. As shown next, this tradeoff yields substantially improved local dynamical behavior while maintaining competitive statistical performance.

\subsection{Local behavior of the learned dynamics}
\label{subsection: local behavior}
Finally, we address the study's key motivation and assess how well the trained models capture the local behavior of the ground-truth dynamics. We evaluate the models on two fronts: (i) the relative vector field error \(\nicefrac{\|f_{\boldtheta}-f\|_2}{|f\|_2}\), which gauges how well the trained models reproduce the ground truth velocities, and (ii) the relative Jacobian error \(\nicefrac{\|\partial f_{\boldtheta}-\partial f\|_F}{|\partial f\|_F}\) as it is the Jacobian of the dynamics that govern the local deformation of the chaotic attractor via the variational equation.

\begin{figure}[ht]
    \centering
    \includegraphics[width=1.0\linewidth]{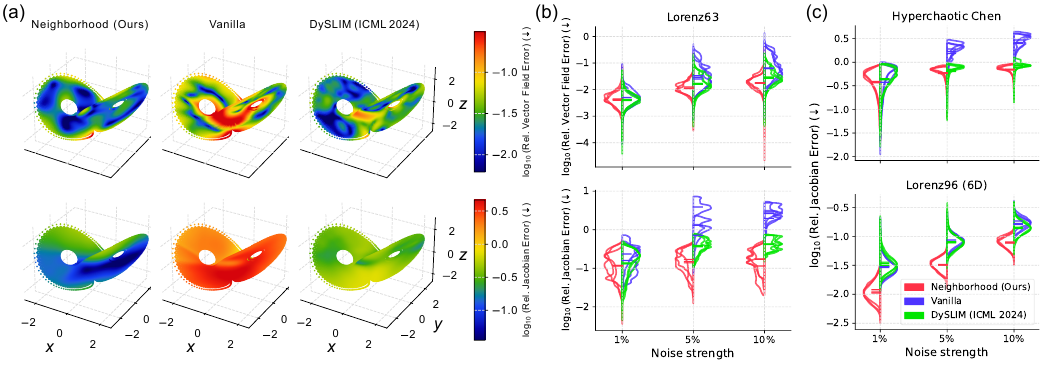}
    \caption{Local behavior of the trained models. \textbf{(a)} Relative error in the vector field (upper row) and the Jacobians (lower row) for models trained on the 3D Lorenz 63 system with 10\% noise. \textbf{(b)} Violin plots of the relative vector field and Jacobian errors for the Lorenz 63 system, across all noise strengths. \textbf{(c)} Relative Jacobian error with varying noise strengths for different systems.}
    \label{fig: local behavior}
\end{figure}

\cref{fig: local behavior}a visualizes these two metrics on the attractor of the Lorenz63 system, from which we immediately see the merits of "watching the neighbors" during training. In particular, we find that while the Vanilla method learns a moderately accurate vector field over most of the attractor, the corresponding Jacobians are completely erroneous, highlighting the fact that learning to reproduce the values of a target function does not guarantee accurate estimation of its derivatives \citep{he2019,latremoliere2022,eisen2025}. Although DySLIM produces more accurate vector fields and Jacobians than Vanilla, our Neighborhood method consistently outperforms DySLIM across all experimental settings (\cref{fig: local behavior}b, c).


We particularly highlight the vector field predictions for the Lorenz63 system, where both Vanilla and DySLIM fail to accurately reconstruct the dynamics near the origin (\cref{fig: local behavior}a, upper row; red patches at the center of the Vanilla and DySLIM attractors). This region is especially challenging, as it corresponds to the lobe-switching mechanism --the trajectory deciding whether to evolve toward the left or right wing of the attractor-- and is therefore associated with strong local instability and unpredictability \citep{dong2025,ayers2023}. Remarkably, the Neighborhood method largely overcomes this difficulty, producing substantially more accurate and spatially uniform vector field predictions across the entire attractor, including the highly sensitive switching region.

\paragraph{Lyapunov exponents} As a key application of accurate Jacobian learning, we assessed the model predictions of the Lyapunov exponents, which is a key descriptor of chaotic dynamics. Due to the page constraint, we present the results in \cref{appendix: more results}, where we show that our model also yields greatly improved values for these exponents.

\section{Related Works}
\paragraph{Estimating the Jacobian from time series data.}
\label{section: jacobian estimation}
Due to its role in Lyapunov exponent estimation and chaos detection, estimating the Jacobian of the time-\(T\) map \(\phi_{\Delta t}\) from data has been extensively studied in the literature. These approaches can broadly be classified as either global or local \citep{mccaffrey1992,escot2023}.

Global methods estimate the Jacobian by fitting surrogate models --such as polynomials \citep{ataei2003}, radial basis functions \citep{potts1991,mccaffrey1992,holzfuss2006,sun2012}, or shallow neural networks \citep{mccaffrey1992,gencay1992,nychka1992,shintani2004}-- to time-series data and subsequently differentiating the learned model\footnote{In this sense, Jacobian estimation via Vanilla Neural ODE training is itself a global method.}. However, as with Vanilla training, accurate trajectory prediction does not necessarily imply accurate derivative estimation \citep{he2019,latremoliere2022}.

In contrast, local methods \citep{eckmann1986,brown1991,abarbanel1992} estimate Jacobians by tracking the evolution of phase space neighborhoods. Around each point of interest, a local regression model --e.g., linear, quadratic \citep{brown1991,abarbanel1992}, or kernel-based \citep{escot2023}-- is fitted to the evolution of neighbor distances, and the Jacobian is inferred from the model’s linear terms. While these approaches directly estimate the Jacobian, training a separate model at each measurement point is computationally and memory intensive. Our Neighborhood method provides a middle ground between global and local approaches by training a single global model constrained by neighborhood-distance evolution for direct Jacobian estimation.


\section{Conclusion}
In this work, we introduced the Neighborhood training method for Neural ODEs, a noise-robust approach that jointly achieves long-term statistical accuracy and accurate local dynamics. Our method is formulated by analyzing the evolution of local covering sets of the attractor (i.e., phase-space neighborhoods), and experiments --that resemble real-world situations-- show substantially improved local dynamical accuracy while maintaining competitive statistical performance.

Our work lays the theoretical, practical and algorithmic foundation for training Neural ODEs with local phase-space information, while leaving several directions for future research, including extensions to higher-dimensional and partially observed systems. The construction of attractor covering sets could also be further optimized, for example by accounting for variations in local predictability across the attractor \citep{abarbanel1991,bailey1997,ziehmann2000}. We nevertheless hope this work provides a complementary starting point for developing more sophisticated methods for learning chaotic dynamics from data.
\begin{ack}

This work was supported by  KIAS Individual Grants [Nos AP102401(J.-H.K.), CG086102 (A.G.) and CG079902 (D.-S.L.)] at Korea Institute for Advanced Study.


\end{ack}

\bibliography{references}
\bibliographystyle{abbrv}

\newpage
\appendix
\section{Additional results}
\label{appendix: more results}
\paragraph{Lyapunov exponents} One key application of accurate Jacobian learning is the estimation of the Lyapunov exponents from noisy trajectory data. This is because Lyapunov exponents characterize the average contraction and expansion rates on the chaotic attractor, serving as a key descriptor of chaotic systems. However, accurately inferring these quantities from surrogate models is challenging, as their computation requires both accurate long-term physical measures and local Jacobian information.
\begin{figure}[ht]
    \centering
    \includegraphics[width=1.0\linewidth]{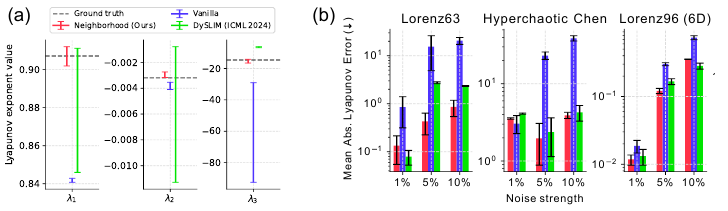}
    \caption{Lyapunov exponents of the trained models. \textbf{(a)} Predictions for the three Lyapunov exponents of the Lorenz63 system, from models trained on 10\% noise data. \textbf{(b)} Mean absolute error of the computed Lyapunov exponents, across systems and noise levels.}
    \label{fig: lyapunovs}
\end{figure}

As shown in \cref{fig: lyapunovs}b, we find that the vanilla method struggles to reproduce the Lyapunov exponents in the presence of larger noise. This is particularly pronounced for the negative exponent(s) (\cref{fig: lyapunovs}a; rightmost panel), which are known to be more difficult to recover from data, as chaotic attractors tend to be very "thin" in directions associated with these exponents \citep{brown1991}. Compared to Vanilla, DySLIM returns more accurate exponents because it estimates the physical measure more accurately. However, our Neighborhood yields the most accurate predictions for the Lyapunov exponents, as it properly captures both ingredients --the global physical measure and the local Jacobians-- that give rise to them.

\section{Data generation}
\label{appendix: data generation}
\subsection{Overview}
For each ODE system considered in this paper, we generated train, validation and test datasets in the following manner:
\begin{enumerate}
    \item Sample \(m=50\) random initial points on the attractor
    \item Estimate the characteristic time scale \(\tau\) of the system, via the average frequency spectrum.
    \item Simulate \(m\) on-attractor trajectories in the time span \([0, 10\tau]\) with \(\Delta t=0.01\tau\)
    \item Shift and scale datasets using the train dataset statistics.
\end{enumerate}

\paragraph{Sampling on-attractor initial points}
To sample points on the chaotic attractor, we first drew \(m\) initial conditions \(\{\boldu_0^{i}\}_{i=1}^m\) from a normal distribution, with the mean and standard deviation chosen so that the sampled points would lie in the basin of attraction of the chaotic attractor. Afterwards, the dynamical equations were numerically integrated up to some prespecified burn-in time \(t_{burn\text{-}in}\) for the trajectories to converge onto the chaotic attractor. 

\paragraph{Characteristic time scale estimation} Using the previously obtained initial conditions \(\{\boldu_0^{i}\}_{i=1}^m\), we generated \(m\) long time trajectories, which were then Fourier transformed along each dimension to yield \(m\) sets of \(d\) frequency spectra. These were then averaged across trajectories to yield a single set of averaged frequency spectra. After a slight smoothing with a 1D Gaussian filter to remove spurious peaks, \texttt{scipy.signal.find\_peaks} function \citep{virtanen2020} was used to identify the most prominent frequency peak for each dimension. Finally, the characteristic time scale \(\tau\) was chosen as the median of the reciprocals of the discovered prominent frequencies per dimension.

\paragraph{Trajectory simulation}
We sample trajectories using the previously found \(m\) random initial conditions on the attractor  \(\{\boldu_0^{i}\}_{i=1}^m\). We normalize the simulation time and the sampling interval using the characteristic time scale \(\tau\) to account for variations in the intrinsic timescales of different ODEs. 

All numerical integration of the ODEs were performed using the \texttt{Tsit5} solver of the \texttt{diffrax} \citep{kidger2022} library, which is an implementation of Tsitouras's 5/4 method \citep{tsitouras2011}. Adaptive time stepping was used to march the solutions forward in time, with relative and absolute tolerance values of 10\textsuperscript{-8} and 10\textsuperscript{-10}, respectively.

\paragraph{Dataset normalization}
It is known that the differing size scales of the ODE dimensions can adversely affect neural ODE training \citep{kim2021}. Therefore, we shifted and scaled the train dataset to have zero mean and unit standard deviation and used the same statistics to normalize the validation and test datasets. 

Note that transforming the data in such a way changes the corresponding ground-truth ODE of a dataset. For an ODE and a bijection 
\begin{equation}
    \frac{d\boldu}{dt} = f(\boldu), \quad \bm{v} = g(\bm{u})
\end{equation}
the transformed state \(\bm{v}\) evolves according to 
\begin{equation}
    \frac{d\bm{v}}{dt}=\partial g|_{g^{-1}(\bm{v})}f(g^{-1}(\bm{v})).
        \label{eqn: shift and scaled ODE}
\end{equation}
This effect was accounted for when computing the relative vector field and Jacobian errors in \cref{subsection: local behavior}.

\subsection{Per-system settings}
Here, we describe the dynamical equations of the systems used, as well as any system-specific parameters for data generation.

\paragraph{Lorenz63 system} The Lorenz63 system \citep{lorenz1963} is given by the following equations:
\begin{equation}
    \frac{d}{dt}\begin{bmatrix}x\\ y\\ z \end{bmatrix} = \begin{bmatrix}
        \sigma (y-x)\\ x(\rho-z)-y\\ xy-\beta z 
    \end{bmatrix}.
\end{equation}
We use the canonical parameter values \((\rho, \sigma, \beta)=(10, \frac{8}{3}, 28)\). The random initial conditions were sampled from the standard normal distribution, and following \cite{froyland1984}, we used a burn-in time of \(t_{burn\text{-}in}=50\).

\paragraph{Hyperchaotic Chen system} The hyperchaotic Chen system was introduced in Li et al. \citep{li2005}, as a modified version of the chaotic Chen system \citep{chen1999}. Hyperchaos means that the system has two or more positive Lyapunov exponents, in contrast to chaotic systems, which only have a single positive exponent. The governing equations are given by
\begin{equation}
    \frac{d}{dt}\begin{bmatrix}x\\ y\\ z \\ w\end{bmatrix} = \begin{bmatrix}
        a(y-x)+w\\ x(d-z)+cy\\xy-bz\\yz+rw 
    \end{bmatrix}.
\end{equation}
We use the parameters \((a, b, c, d, r)=(35, 3, 12, 7, 0.58)\), which are the parameters reported in the original paper, to give rise to hyperchaos. For the random initial conditions, we use samples from \(\mathcal{N}([0, 0, 20, 0]^{\mathsf{T}}, I_{4\times 4})\), which we found allows the trajectories to converge faster onto the attractor. The burn-in time was set to \(t_{burn\text{-}in}=100\).

\paragraph{Lorenz96 system}
The Lorenz96 system \citep{lorenz1996} is system of ODEs defined on a periodic lattice of \(d\) nodes, given by
\begin{equation}
    \frac{dx_i}{dt}=(x_{i+1}-x_{i-2})x_{i-1}-x_i+F, \quad i=0,\dots,d-1.
\end{equation}
For large values of the external force \(F\), this system displays chaotic behavior. In our experiments, we used \(F=10\) and \(d=6\). We drew random initial conditions from the standard normal distribution and used a burn-in time of \(t_{burn\text{-}in}=1000\).

\section{Hyperparameter selection}
\label{appendix: hyperparameter selection}
To choose the hyperparameters for our main experiments, we performed a small grid search using a single random seed. Hyperparameter sweep was performed for each dataset and noise level, and we detail our selection procedure below and present the selection results in the following subsection.

\subsection{Selection criteria overview}

We designed our hyperparameter selection criteria to \textbf{prefer models with long-term statistical accuracy}. This was because hyperparameter combinations with the strongest short-term performances tended to have overfit on the short-term dynamics, especially for experiments with larger noise levels. Alternatively, choosing hyperparameters based on the accuracy of the learned Jacobians also had the problem of unnecessarily skewing the results in our favor, as the baseline DySLIM method was not designed for this task. As both our Neighborhood method and DySLIM were both developed with the statistical accuracy of the trained models in mind, selecting hyperparameters based on long-term statistical accuracy ensured that both training methods had a chance of showcasing their efficacy.

To attribute a score to each hyperparameter combination in the grid search, we computed the Sinkhorn divergence \citep{cuturi2013,genevay2018} --which is an optimal transport-based metric that also gauges the discrepancy between two measures using point samples-- between the model predictions and the ground truth trajectories in the validation dataset. Here, we used the Sinkhorn divergence instead of the \(\mmd\) to avoid optimizing the hyperparameter selection by the exact metric used for subsequent model evaluations.

Sinkhorn divergences were computed using the \texttt{ott-jax} package \citep{cuturi2022} using trajectories from the validation dataset. As the 50 trajectory samples of length \(10\tau\) was too small to produce a reliable estimate of the Sinkhorn divergence, we first cut each trajectory into 10 pieces to produce 500 segments of length \(\tau\). Subsequently, the state values at the start of each trajectory segments were used as initial conditions to the trained Neural ODEs to produce 500 segment predictions. Finally, the Sinkhorn divergence between the data values and model predictions at the end of the segments (corresponding to the 500 point samples of the attractor evolved for time \(\tau\)) was calculated using the default parameter settings of the \texttt{ott-jax} library, and used as the score for a given hyperparameter combination.

We comment that for this Sinkhorn divergence computation, the noiseless version of the validation dataset was used, contrary to model checkpointing where the noisy validation dataset was used. This choice was to decouple the effects of the hyperparameter selection process from the actual algorithm performances. Otherwise, the experimental results would measure the adequacy of the hyperparameter selection criteria in identifying the best performing models for each training method, as opposed to gauging the actual algorithm performances. Indeed, the original DySLIM paper \citep{schiff2024} also selects hyperparameters based on model performances on clean datasets, as the work does not consider noise effects. Nonetheless, a complete deployment of both our Neighborhood method and DySLIM to noisy, real world datasets would require a more effective hyperparameter selection strategy, which we leave for future work.

\subsection{Hyperparameter selection results}

\paragraph{DySLIM} DySLIM has two hyperparameters \(\lambda_1\) and \(\lambda_2\), corresponding to the weights of its two additional loss functions. We used the same grid used in the original paper \citep{schiff2024} and performed a grid search over \(\lambda_1\in[0,1,10]\) and \(\lambda_2\in[1,10,100,1000]\).

\paragraph{Neighborhood} Our Neighborhood method also has two parameters: \(K\), which is the number of neighbors sampled from each local neighborhood per training step, and \(\lambda\), which is the weight factor for our neighborhood-based regularization term (\cref{eqn: neighborhood loss function}). We performed a grid search over \(K\in[8,16,32,64]\) and \(\lambda\in [1,10,100,1000]\).

\begin{figure}[htp]
    \centering
    \includegraphics[width=1.0\linewidth]{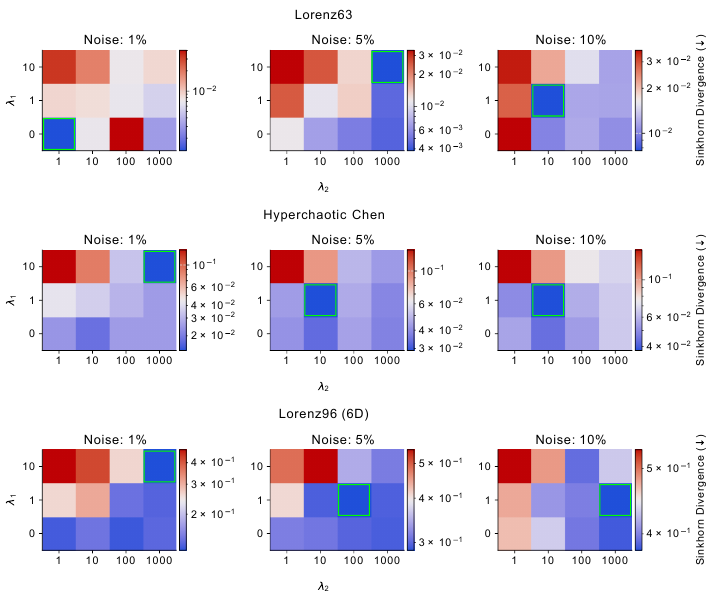}
    \caption{Hyperparameter sweep results for DySLIM. The green box indicates the selected hyperparameter combination.}
    \label{fig: dyslim hyperparameter sweep}
\end{figure}

\begin{figure}[htp]
    \centering
    \includegraphics[width=1.0\linewidth]{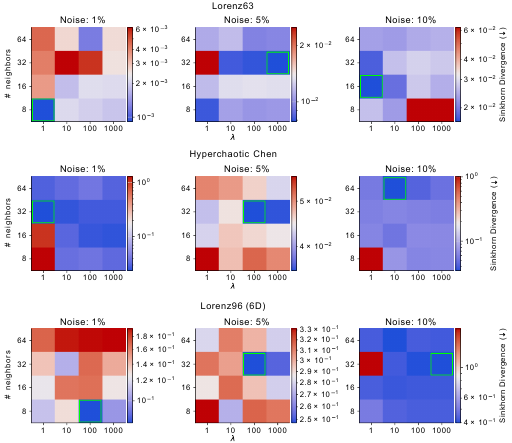}
    \caption{Hyperparameter sweep results for our Neighborhood method. The green box indicates the selected hyperparameter combination.}
    \label{fig: neighborhood hyperparameter sweep}
\end{figure}

\section{Evaluation metrics}
\label{appendix: evaluation metrics}
As our goal is to assess how closely the models recover the ground truth dynamics, even when the training data is corrupted by noise, we perform all evaluations on the noiseless version of the test dataset (e.g., trajectory errors are computed with respect to the clean trajectories and not the noisy ones). All metrics, except for \(\mmd\) is computed using the 50 trajectories in the test dataset.

\subsection{Valid prediction time} To evaluate the short-term accuracy of the model predictions, the first metric we introduce is the normalized root mean squared error between the data and model predictions, which for a trajectory prediction starting at the initial condition \(\boldu_0^i\) is given as,
\begin{equation}
    \text{NRMSE}^i(t_j)=\frac{1}{\sqrt{d}}||\phi_{j\Delta t;\boldtheta}(\boldu_0^i)-\boldu_j^i||_2^2.
\end{equation}
Note that while the original NRMSE requires scaling the error \(\phi_{j\Delta t;\boldtheta}(\boldu_0^i)-\boldu_j^i\) with the standard deviation of the data trajectory before taking the norm, it is not necessary in our case as the data is affine transformed to have unit standard deviation along each dimension beforehand. For plots like \cref{fig: short-term prediction}a, these per trajectory NRMSEs can be averaged to yield the mean NRMSE:
\begin{equation}
    \text{NRMSE}(t_j)=\frac{1}{m}\sum_{i=1}^m\text{NRMSE}^i(t_j).
\end{equation}

The valid prediction time is defined as the maximum time until which the NRMSE value stays below a user-determined threshold \(\epsilon\), normalized by the maximum Lyapunov exponent \(\lambda_1\):
\begin{equation}
    \text{VPT}^i = \lambda_1\cdot \max \{T \in \mathbb{R}\ |\ \text{NRMSE}^i(t) \leq \epsilon,\ \forall t \leq T\}.
\end{equation}
To yield a single representative value for each experiment, these per-trajectory VPT values were also averaged as,
\begin{equation}
    \text{VPT}=\frac{1}{m}\sum_{i=1}^m \text{VPT}^i.
\end{equation}
For our analyses, the VPT threshold was set to \(\epsilon = 0.3\).

\subsection{Maximum mean discrepancy}
The maximum mean discrepancy (MMD) is a metric that quantifies the difference between two measures \(\mu\) and \(\nu\) under mild assumption \cite{gretton2012}. A big advantage of MMD is that it can be efficiently estimated using samples from the measures of interest. For samples \(\{\boldu^k\}_{k=1}^K\sim\mu\) and  \(\{\bm{v}^k\}_{k=1}^K\sim\nu\), the squared MMD can be estimated using the sample as \citep{gretton2012},
\begin{equation}
    \widehat{\text{MMD\textsuperscript{2}}}(\mu,\nu)=\frac{1}{K(K-1)}\sum_{k=1}^K\sum_{k'\neq k}^K\left[ \kappa\left(\boldu^k,\boldu^{k'}\right)+\kappa\left(\bm{v}^k,\bm{v}^{k'}\right)\right]-\frac{2}{K^2}\sum_{k=1}^K\sum_{k'=1}^K \kappa\left(\boldu^k,\bm{v}^{k'}\right)
\end{equation}
where \(\kappa(\cdot,\cdot)\) is a suitable kernel function. For our analyses, we used the mixture of rational quadratic kernels \citep{schiff2024}
\begin{equation}
    \kappa(\boldu,\bm{v})=\sum_{\sigma_q \in \bm{\sigma}}\frac{\sigma_q^2}{\sigma_q^2+||\boldu-\bm{v}||_2^2};\quad \bm{\sigma}=\{0.2, 0.5,0.9,1.3\}.
\end{equation}
In particular, since $\kappa$ is a characteristic kernel, the condition $\text{MMD\textsuperscript{2}}=0$, implies $\mu=\nu$.

For the \(\mmd\) results, we made a separate test dataset of 5000 points on the attractor that are evolved for a longer period in time, as more samples and a longer time integration were needed in order to properly assess the quality of the learned measures. The on-attractor points were generated analogously to the data generation process, and the integration time was taken to be \([0,100\lambda_1^{-1}]\) where the maximum Lyapunov exponent of each dataset was determined beforehand as according to \cref{appendix: Lyapunov exponents}.

The baseline values for \(\mmd\) characterizes the value between two independent samples of the chaotic attractor that are both evolved according to the ground truth equations. This value corresponds to the best possible attainable value, when the model is able to exactly reproduce the long-term physical measure of the ground truth system. This baseline \(\mmd\) value was computed by taking two 5000 point samples on the ground truth attractor, evolving the both with the ground truth dynamics, then computing the \(\mmd\) for each time instant. Subsequently, the \(\mmd\) values for all time points were averaged to yield a single baseline value \cref{fig: mmd baseline}.

\begin{figure}[ht]
    \centering
    \includegraphics[width=1.0\linewidth]{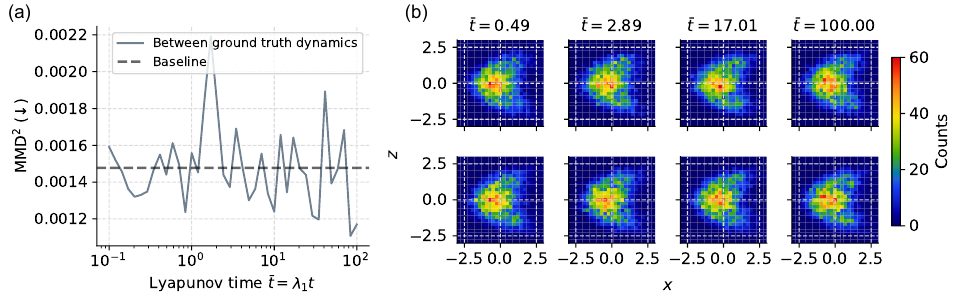}
    \caption{Calculation of baseline \(\mmd\) value for the hyperchaotic Chen system. \textbf{(a)} Time evolution of the \(\mmd\) computed between two samples of the chaotic attractor, both evolved using the ground truth dynamics. Note that unlike \cref{fig: statistical accuracy}b, the values do not start from zero as the initial conditions for the two attractor samples are different. \textbf{(b)} Histograms of the ground truth attractor samples at different rollout times. The two rows correspond to the two attractor samples that are evolved in time.}
    \label{fig: mmd baseline}
\end{figure}

\subsection{Relative vector field / Jacobian error}
To evaluate the accuracy of the learned vector fields, we computed the relative vector field error for each point in the test dataset, as defined below:
\begin{equation}
    \text{Rel. Vector Field Error}(\boldu)=\frac{||f_{\boldtheta}(\boldu)-f(\boldu)||_2}{||f(\boldu)||_2},
\end{equation}
where we used the shifted and scaled version of the ground truth equations (\cref{eqn: shift and scaled ODE}) to supply the values of \(f(\boldu)\).

Similarly, the accuracy of the learned Jacobians at each point was assessed using the relative Jacobian error, defined as:
\begin{equation}
    \text{Rel. Jacobian Error}(\boldu)=\frac{||\partial f_{\boldtheta}(\boldu)-\partial f(\boldu)||_F}{||\partial f(\boldu)||_F},
\end{equation}
where we once again accounted for the dataset normalization for the computation of the ground truth Jacobians.

\subsection{Lyapunov exponents}
\label{appendix: Lyapunov exponents}
The Lyapunov exponents characterize the asymptotic contraction and expansion rates of the chaotic attractor, and are given as the spectrum of the matrix \citep{oseledec1968,christiansen1997,guckenheimer1983}
\begin{equation}
    \Lambda=\lim_{t\rightarrow \infty} \frac{1}{2t}\log(M^\mathsf{T}_{\boldu}(t) M_{\boldu}(t)).
\end{equation}

where \(M_{\boldu}(t)\) is the fundamental matrix solution of the variational equation (\cref{eqn: variational equation}, second line).

While one can attempt to directly solve \cref{eqn: variational equation}, this procedure is numerically unstable as  \(M_{\boldu}(t)\) quickly becomes singular in time. Instead, we use the algorithm of Shimada et al. \citep{shimada1979}, where for each initial condition \(\boldu_0^i;\; i=1,\dots, m\), the algorithm produces an estimate of the Lyapunov spectrum \(\{\lambda_r(\boldu_0^i)\}_{r=1}^d\) by solving the \(d(d+1)\) coupled ODEs corresponding to the original ODE and its variational equation
\begin{equation}
\begin{aligned}
&\displaystyle \frac{d \boldu (t)}{d t} = f(\boldu (t)), \quad &\boldu (0)=\boldu_0\\
&\frac{d \bm{M}_{\boldu}(t)}{dt} = Df|_{\boldu}\bm{M}_{\boldu}(t), \quad &\bm{M}_{\boldu}(0)=\mathbb{I}_d
\end{aligned}
\label{eqn: variational equation}
\end{equation}
over the time interval \([0, T]\) and performing  Gram-Schmidt orthonormalization at fixed time intervals \(\Delta t\) to keep the fundamental matrix \(\bm{M}_{\boldu}(t)\) nonsingular. \cref{eqn: variational equation} was solved using the same solver settings as data generation: \texttt{Tsit5} solver from the \texttt{diffrax} \citep{kidger2022} library, with relative and absolute tolerance values of 10\textsuperscript{-8} and 10\textsuperscript{-10}, respectively.

To compute the mean absolute error of the Lyapunov spectrum, we first produced a mean Lyapunov spectrum estimate by averaging over the results from different initial conditions:
\begin{equation}
    \bar{\lambda}_r = \frac{1}{m}\sum_{i=1}^m \lambda_r(\boldu_0^i).
    \label{eqn: mean lyapunov}
\end{equation}
Denoting \(\bar{\lambda}_{r;\boldtheta}\) as the mean Lyapunov exponents computed using the trained models, mean absolute Lyapunov error was then computed as,
\begin{equation}
    \text{Mean Abs. Lyapunov Error}=\frac{1}{d}\sum_{r=1}^d\left|\bar{\lambda}_{r;\boldtheta}-\bar{\lambda}_r\right|.
\end{equation}

\section{Proof of Theorem 1}
\label{appendix: theorem proof}

\begin{theorem*}
Let $(X,\mathcal{C})$ be a measurable space, and let $\nu_1,\nu_2$ be finite measures on $(X,\mathcal{C})$ such that $ \nu_1(X\setminus \mathcal{A})=\nu_2(X\setminus \mathcal{A})=0 $
for a set $\mathcal{A}\in\mathcal{C}$ (both measures have support on $\mathcal{A}$). 

Furthermore, let $\{B_j\}_{j\in\mathbb{N}}\subset\mathcal{C}$ be a countable cover of $\mathcal{A}$, and assume that $\nu_1|_{B_j}=\nu_2|_{B_j}$, $\forall j$, that is
\[
\nu_1(C)=\nu_2(C)\quad\text{for all measurable }C\subset B_j,\;\forall j.
\]
Then $\nu_1=\nu_2$ on $\mathcal{C}$.
\end{theorem*}

\begin{proof}
Let $A\in\mathcal{C}$. Since both measures vanish outside $\mathcal{A}$,
\[
\nu_k(A)=\nu_k(A\cap \mathcal{A}),\quad k=1,2.
\]
Thus it suffices to compare $\nu_1$ and $\nu_2$ on $A\cap \mathcal{A}$. Define
\[
C_1 := A\cap \mathcal{A}\cap B_1,\qquad
C_n := A\cap \mathcal{A}\cap B_n \setminus \bigcup_{k=1}^{n-1} B_k,\quad n\ge2.
\]
Then $\{C_n\}$ are pairwise disjoint, $C_n\subset B_n$, and since the $B_j$ are a cover we have that
\[
A\cap \mathcal{A} = \bigcup_{n=1}^\infty C_n.
\]
By the assumption, $\nu_1(C_n)=\nu_2(C_n), \forall n.$ Using countable additivity, 
\begin{equation*}
    \nu_1(A)=\sum_{n=1}^\infty \nu_1(C_n)
=\sum_{n=1}^\infty \nu_2(C_n)
=\nu_2(A).
\end{equation*} Since $A$ was arbitrary, $\nu_1=\nu_2$.
\end{proof}

\end{document}